\documentclass{article}

\usepackage{balance} 

\usepackage[margin=3cm]{geometry}

\usepackage{amsmath,amssymb,amsthm,mathtools}
\usepackage[utf8]{inputenc}
\usepackage{xspace}
\usepackage{url}
\usepackage{dsfont}
\usepackage{enumerate}

\usepackage{caption}
\usepackage{subcaption}

\usepackage{balance}
\usepackage[shortcuts]{extdash}

\usepackage{enumitem}

\usepackage[numbers,longnamesfirst]{natbib}

 \newtheorem{theorem}{Theorem}
 \newtheorem{definition}[theorem]{Definition}
 
 \newtheorem{lemma}[theorem]{Lemma}
 \newtheorem{corollary}[theorem]{Corollary}
 
 \newtheorem{remark}[theorem]{Remark}
 \numberwithin{theorem}{section}

\usepackage{tikz}
\usetikzlibrary{shapes,arrows}
\usetikzlibrary{decorations}
\usetikzlibrary{plotmarks}
\usetikzlibrary{calc}
\usetikzlibrary{mindmap}
\usetikzlibrary{shadows}
\usetikzlibrary{backgrounds}
\usetikzlibrary{shapes}
\usetikzlibrary{shapes.symbols}

\usepackage[algo2e,ruled,vlined,linesnumbered]{algorithm2e}
\SetAlgoSkip{}

\allowdisplaybreaks[3]

\newcommand*{\om}{\textsc{OneMax}\xspace}

\newcommand*{\onemax}{\om}
\newcommand*{\hottopic}{\textsc{HotTopic}\xspace}

\newcommand*{\jump}{\textsc{Jump}\xspace}
\newcommand*{\leadingones}{\textsc{LeadingOnes}\xspace}

\DeclareMathOperator{\Prob}{Pr}

\newcommand*{\E}{\mathrm{E}}

\newcommand{\N}{\mathds{N}}

\newcommand{\eps}{\varepsilon}
\newcommand{\wellbehaved}{di\-ver\-si\-ty\-/neu\-tral\xspace}

\newcommand{\ones}[1]{\left|#1\right|_1}
\newcommand{\zeros}[1]{\left|#1\right|_0}

\newcommand{\muea}{($\mu$+1)~EA\xspace}
\newcommand{\muga}{($\mu$+1)~GA\xspace}

\newcommand{\cross}{\mathrm{c}}
\newcommand{\op}{\mathrm{op}}

\newcommand{\toggleplot}[1]{{\textcolor{red}{Plots removed to increase compilation speed. Use command $\backslash$toggleplot in preamble to reinsert them.}}}

\newcommand{\added}[1]{\textcolor{black}{#1}}

\author{Johannes Lengler\\
  ETH Z\"urich\\
  Z\"urich, Switzerland
  \and Andre Opris\\
  University of Passau\\
  Passau, Germany
  \and Dirk Sudholt\\
  University of Passau\\
  Passau, Germany
  }

\title{Analysing Equilibrium States for Population Diversity}

\begin{document}

\maketitle

\begin{abstract}
Population diversity is crucial in evolutionary algorithms as it helps with global exploration and facilitates the use of crossover. Despite many runtime analyses showing advantages of population diversity, we have no clear picture of how diversity evolves over time.

We study how population diversity of $(\mu+1)$ algorithms, measured by the sum of pairwise Hamming distances, evolves in a fitness-neutral environment. 
We give an exact formula for the drift of population diversity and show that it is driven towards an equilibrium state. Moreover, we bound the expected time for getting close to the equilibrium state.   
We find that these dynamics, including the location of the equilibrium, are unaffected by surprisingly many algorithmic choices. All unbiased mutation operators with the same expected number of bit flips have the same effect on the expected diversity. Many crossover operators have no effect at all, including all binary unbiased, respectful operators. We review crossover operators from the literature and identify crossovers that are neutral towards the evolution of diversity and crossovers that are not.
\end{abstract}



\textbf{Keywords:} Runtime analysis, diversity, population dynamics

\section{Introduction and Motivation}

Population diversity is an important aspect of evolutionary algorithms~\cite{Squillero2016,Glibovets2013,Crepinsek2013,Shir2012}. A diverse set of solutions helps with exploration, is important for escaping local basins of attraction, and is the basis for efficient use of crossover operators~\cite{sudholt2020benefits}. Several studies confirmed the benefits of explicit diversity-preserving mechanisms on various test functions~\cite{Friedrich2009,Oliveto2018,CovantesOsuna2019runtime,Covantes2019,CovantesOsuna2021,Oliveto2015,Lissovoi2017a}. For some operators like lexicase selection, it is even known that diversity decreases the runtime of this operator~\cite{Helmuth22}.

Many theoretical and practical results show that even low levels of population diversity can improve runtime. Even on the most simple benchmark \onemax, the standard $(2+1)$~GA (with mutation rate $1/n$) is by a constant factor faster than the fastest mutation-based evolutionary algorithm without crossover~\cite{Sudholt2016,Corus2017b,CorusO20}. This is due to the beneficial effects of crossover, which can exploit even small amounts of diversity. For the more complex monotone function \hottopic, \added{the same effect reduces the exponential optimisation time of the \muea to $O(n\log n)$ for the \muga if $\mu$ is a large constant and the algorithms are }started close to the optimum~\cite{lengler2019general}. Finally, it was also shown to benefit memetic (hybrid) evolutionary algorithms on \textsc{Hurdle} functions~\cite{Nguyen2019}.

Examples where crossover between more diverse individuals can help include \textsc{Real Royal Road} functions~\cite{Jansen2005c} and $\jump_k$. For $\jump_k$, it is necessary to cross a fitness valley of size $k$. The \muga can do this with crossover in time $O(4^k)$ if sufficiently diverse individuals exist, while mutation-based operators need $\Omega(n^k)$ trials~\cite{Jansen2002}. However, the original approach by Jansen and Wegener, later improved by K\"otzing, Sudholt and Theile, only showed that sufficiently diverse individuals appear for unrealistically small crossover probability~\cite{Jansen2002,Koetzing2011a}. 
Dang et al.~\cite{Dang2017} showed that a more modest improvement of roughly a factor $n$ is still possible when always performing crossover. This study showed that diversity emerges naturally on the set of all search points with $n-k$ ones, and that on this set crossover serves as a catalyst for boosting population diversity. 
However, the full benefits of crossover can still be obtained if the \muga is equipped with diversity-preserving mechanisms~\cite{Dang2016}.

So there is no shortage of results showing that diversity can be beneficial. Despite these results, our understanding of how population diversity evolves is still very limited. Even on \onemax, for a standard \muea, we do not have a complete picture. While there are upper bounds whose leading constants decrease with~$\mu$ up to $\mu = o(\sqrt{\log n})$~\cite{CorusO20}, lower bounds that are tight including leading constants are only known for $\mu=2$~\cite{Oliveto2022}. For $\jump_k$, empirical results in~\cite{Dang2017} suggest that the improvement by crossover is much larger than the mentioned factor of~$n$ from the theoretical analysis~\cite{Dang2017}. In both scenarios, the main obstacle is understanding the population dynamics and the evolution of diversity.

When considering problems with large degrees of \emph{neutrality}, that is, contiguous regions of the search space (with respect to the Hamming neighbourhood) of equal fitness, or plateaus in the fitness landscape, our understanding of population diversity is also not well developed. Many important problems feature neutrality, and functions with plateaus have been analysed in the literature in the context of runtime analysis of evolutionary algorithms. This includes, for example, (1) the \emph{hidden subset problem}~\cite{cathabard2011non,doerr2017unknown,doerr2019solving,einarsson2019linear}, where the fitness only depends on a small fraction of all variables, and it is not known which variables are relevant and which ones only lead to neutral changes, (2) \emph{majority functions} returning the majority bit value~\cite{Bian2020,DoerrKrejcaArXiv2022}, (3) the \emph{moving Hamming ball} benchmark~\cite{DangJL17} from dynamic optimisation where a Hamming ball around a moving target must be tracked and the fitness areas within and outside of the Hamming ball are both flat, and (4) the \textsc{Plateau}$_k$ function~\cite{AntipovD21,Eremeev20}, a variant of \onemax in which the best $k$ fitness levels are turned into a neutral region, except for the optimum at $\vec{1}$. However, except for~\cite{AntipovD21,Eremeev20} the above results either concern populations of size~1 or do not give detailed insights into the diversity of the population. The aforementioned work on \jump~\cite{Dang2017} does give insights into the population diversity as part of the analysis, however these insights are limited to the specific set of search points with $n-k$ ones.

We aim to initiate the systematic theoretical analysis of population diversity in steady-state algorithms to gain insights into how diversity evolves, how quickly diversity evolves, and which factors play a role in the evolution of diversity. In contrast to previous work, we do not consider functions with specific fitness gradients and take an orthogonal approach. We study how population diversity, defined as the sum of pairwise Hamming distances in the population, evolves in the absence of fitness-based guidance as found in a completely neutral environment, that is, a flat fitness function. 

We consider general classes of \muea{}s and \muga{}s equipped with various mutation and crossover operators. 
As diversity measure $S$, we consider twice the sum of pairwise Hamming distances of population members.
We show that, for all unbiased mutation operators (as will be defined later), the diversity in all \muea{}s is pushed towards an equilibrium state $S_0$ that depends on the population size~$\mu$, the expected number $\chi$ of bits flipped during mutation, and the problem size~$n$:

%
%
%
%
%
\[
    S_0 \coloneqq \frac{(\mu-1)\mu^2 \chi n}{2(\mu-1)\chi + n}.
\]
At this equilibrium the expected Hamming distance between two random population members (with replacement) is roughly $(\mu-1)\chi$ if $2(\mu-1)\chi \ll n$, i.\,e.\ increasing linearly with the population size~$\mu$ and the mutation strength~$\chi$, and roughly $n/2$ if $2(\mu-1)\chi \gg n$, respectively. The term $n/2$ makes sense as this is the expected average Hamming distance in a uniform random population.

We show that, for reasonable parameters, the expected time to decrease the diversity below $(1+\varepsilon) S_0$, with $\varepsilon > 0$ constant, when starting with any larger diversity is bounded by 
$O(\mu^2 \ln n)$. \added{This bound grows very mildly with the problem size~$n$. On the other hand, the expected time to increase diversity above $(1-\varepsilon) S_0$, when starting with less diversity, is 
$O(n\ln n)$, and can thus be larger by a factor $n/\mu^2$ for small values of $\mu$.}


We also show that, surprisingly, the dynamics are to a very large extent independent of the specifics of the algorithm:\footnote{By ``dynamics'', we mean the expected change, the equilibrium value $S_0$, and our upper bounds for the expected time to approach $S_0$.}
\begin{itemize}
    \item For fixed $\chi>0$, every unbiased mutation operator which flips $\chi$ bits in expectation, leads to the same dynamics. For example, standard bit mutation with mutation rate $1/n$ has the same dynamics as RLS mutation using only 1-bit flips.
    \item Large classes of crossover operators, including uniform crossover and $k$-point crossover, have no effect on the dynamics. 
\end{itemize}

For these reasons, we systematically classify which crossover operators have an effect on the dynamics of population diversity. In Section~\ref{sec:equilibria-crossover} we show that crossover operators are neutral with respect to diversity if and only if they satisfy a certain characteristic equation. Consequently, we call such operators \emph{diversity-neutral}. 
In Section~\ref{sec:crossover-operators}, we investigate this property further and show that it is implied if the crossover is respectful with a mask independent of the order of the parents, see Section~\ref{sec:preliminaries} for formal definitions. Moreover, we will show that unbiased crossover operators are diversity-neutral if and only if they are respectful, i.e., if and only if the offspring are in the convex hull of the parents. Finally, in Section~\ref{sec:classification} we apply our classification, building on results from~\citet{Friedrich2022}, to classify five crossover operators from the literature as diversity-neutral, and seven other operators as not diversity-neutral.

\subsection{Motivation for Studying Flat Landscapes}
\label{sec:motivation-for-flat-landscapes}
There are two major motivations for our study of a flat fitness landscape. 
One reason is that, informally, they could provide \emph{upper bounds} on the population diversity that we obtain in many non-neutral environments. While we suspect that counterexamples exist, we also suspect that for many ``natural'' fitness functions, diversity in non-flat environments is generally \emph{lower} than  diversity in flat environments. After all, selection tends to favour individuals which are similar to each other, since it systematically promotes individuals which have a similar trait (namely, high fitness). In contrast, in a flat fitness landscape any new offspring is accepted, allowing the population to spread out without restrictions imposed by the topology of the search space. Thus, there is some hope that the diversity bounds of this paper may still hold as \emph{upper bounds} in many non-neutral environments.

The second reason is that, in addition to \onemax and $\jump_k$ mentioned earlier, there are several processes of interest to the runtime analysis community that feature large degrees of neutrality or very low selective pressure, either continuously or temporarily. 
\begin{itemize}
    \item For the well-known \leadingones function, if the best-so-far fitness value is $k$ then the bits at positions $k+2, \dots, n$ receive no fitness signal and thus this sub-space of the hypercube is a perfectly neutral environment. The dynamics of a \muea or \muga \emph{in accepted steps} are similar to the dynamics studied in the following.
    \item \emph{Clearing}~\cite{CovantesOsuna2019runtime,sudholt2020benefits} is a diversity-preserving mechanism in which an individual of high fitness ``clears'' a region around itself, i.e. the fitness of ``cleared'' individuals is replaced with a plateau of low fitness. The evolution of the population happens on a flat fitness function, except for the fact that winner individuals are guaranteed to survive and continuously spawn offspring close to them.
    \item A similar process can be seen for \muea{}s on \emph{\hottopic} functions. It has been shown that after an improving individual is found, the offspring of this individual may essentially evolve free from selective pressure for a while, as if they were in a fitness-neutral environment. The defects accumulated in this phase cause the $(\mu+1)$ EA to take exponential time on \hottopic if $\mu$ is a large constant~\cite{lengler2021exponential}. 
    \item Selection pressure can also be absent if an evolutionary algorithm uses inappropriate parameter settings or operators that are not suitable. Examples of inefficient parameter settings are given in~\cite{Lehre2010a}. Selection pressure was also found to be nearly absent when using fitness-proportional replacement selection in probabilistic crowding~\cite{Covantes2019} or when using stochastic pure ageing, where individuals are being removed from the population probabilistically~\cite{Oliveto2014}. So our results may be helpful to understand the effects of bad EA designs or parameter choices.
\end{itemize}
 We emphasize that all these scenarios are unique in their own way: fitness plateaus have a topology that may be different from the hypercube; the clearing diversity mechanism continuously injects offspring of the current winners into the population; the phases without selective pressure in \hottopic optimization only last for a certain amount of time. These unique traits do affect the dynamics of population diversity. Therefore, our results only apply partially to those situations. Nevertheless, we believe that our study is a good starting point for better understanding such specific situations.

We remark that in biology, specifically in population genetics, evolution in the absence of selection has been studied as well, justified by the fact that many loci (bits) have little effect on the overall fitness of the organism~\cite[Chapter~3]{kingman1980mathematics} and the hypothesis that evolution is largely driven by genetic drift~\cite{Kimura1979}. According to~\cite{paixao_unified_2015}, the \muea is known in population genetics as the \emph{Moran model} and the diversity measure is known as \emph{gene diversity} according to~\cite{wineberg2003underlying}. Despite these links, the closest related work we were able to identify analyses equilibria for gene frequencies~\cite[Chapter~3]{kingman1980mathematics} known as Hardy-Weinberg equilibria~\cite{Edwards2008}, and nearly all studies consider a fixed, constant number of loci. In contrast, our work deals with equilibria for gene diversity on strings of arbitrary length $n$ (where $n$ is often also used to parameterise the mutation strength). 
Furthermore, our work covers a broader range of mutation and crossover operators, many of which are not found in nature.

\section{Preliminaries}
\label{sec:preliminaries}

\paragraph{Notation:}
For $x,y\in\{0,1\}^n$, the \emph{Hamming distance} $H(x,y)$ is the number of positions in which $x$ and $y$ differ. For $k,\ell\in \N$ with $k\le \ell$ we write $[k]:= \{1,2,\ldots,k\}$ and $[k,\ell] := \{k,\ldots,\ell\}$. By a \emph{flat} (or \emph{fitness-neutral}) fitness function we mean the function $f(x)=0$ for all $x\in \{0,1\}^n$. For $x \in \{0,1\}^n$ we mean by $\ones{x}$ the number of ones and by $\zeros{x}$ the number of zeros, respectively. By $\vec{i} \in \{0,1,2\}^n$ we mean $\vec{i}:=(i,\ldots,i)$ for $i \in \{0,1,2\}$. We write $\ln^+(z) \coloneqq \max\{1,\ln z\}$.

\paragraph{Algorithms:}
We define the following schemes of a steady-state EA without crossover and a steady-state GA using crossover. 
The former starts with some initial population, selects a parent uniformly at random and creates an offspring $y$ through some mutation operator. Then $y$ replaces a worst search point $z$ in the current population if its fitness is no worse than $z$. 

\begin{algorithm2e}[ht]
  $t \gets 0$\\
  Initialise $P_0$ as a multiset of $\mu$ search points.\\
  \While{termination criterion not met}{
    Select $x \in P_t$ uniformly at random.\\
    $y \gets \mathrm{mutation}(x)$.\\
    Select $z \in P_t$ uniformly at random from all search points with minimum fitness in $P_t$.\\
    \If{$f(y) \ge f(z)$}{$P_{t+1} \gets P_t \cup \{y\} \setminus \{z\}$.}
	$t \gets t+1$
}
\caption{Scheme of a Steady-State $(\mu+1)$ EA}
\label{alg:steady-state-EA}
\end{algorithm2e}

The steady-state GA picks two parents uniformly at random with replacement and applies some crossover operator to the two parents, followed by some mutation operator applied to the offspring. The mutant replaces a worst search point from the previous population.

\begin{algorithm2e}[ht]
  $t \gets 0$\\
  Initialise $P_0$ as a multiset of $\mu$ search points.\\
  \While{termination criterion not met}{
    Select $x_1 \in P_t$ uniformly at random.\\
    Select $x_2 \in P_t$ uniformly at random.\\
    $y \gets \mathrm{crossover}(x_1, x_2)$.\label{line:crossover-in-steady-state-GA}\\
    $y' \gets \mathrm{mutation}(y)$.\\
    Select $z \in P_t$ uniformly at random from all search points with minimum fitness in $P_t$.\\
    \If{$f(y') \ge f(z)$}{$P_{t+1} \gets P_t \cup \{y'\} \setminus \{z\}$.}
	$t \gets t+1$
}
\caption{Scheme of a Steady-State $(\mu+1)$ GA}
\label{alg:steady-state-GA}
\end{algorithm2e}

We deliberately do not specify operators for initialisation, crossover and mutation at this point to obtain a scheme that is as general as possible. 
%
Note that Algorithm~\ref{alg:steady-state-GA} chooses two parents with replacement. It is straightforward to adapt our results to selecting parents without replacement (that is, ensuring that two \emph{different} parents are recombined), see Remark~\ref{rem:no-replacement} in Section~\ref{sec:equilibria-crossover}. Parent selection is assumed to be uniform. For our setting, this is no restriction: assuming the fitness function is flat and ties are broken uniformly at random, every selection method based on fitness values or rankings of search points boils down to uniform selection.

Algorithm~\ref{alg:steady-state-GA} is a generalisation of Algorithm~\ref{alg:steady-state-EA}: if we choose a crossover operator that returns an arbitrary parent (called \emph{boring crossover} in~\citep{Friedrich2022}), the algorithm performs a mutation of a parent chosen uniformly at random as in Algorithm~\ref{alg:steady-state-EA}.
It is also straightforward to implement a crossover probability $p_c$, that is, to apply some crossover operator $\cross$ 
with probability $p_c$. In this case the crossover operator in Line~\ref{line:crossover-in-steady-state-GA} of Algorithm~\ref{alg:steady-state-GA} performs a boring crossover with probability $1-p_c$ and otherwise executes $\cross$.

In both schemes, in case of a fitness tie between the offspring and~$z$, the offspring is preferred. This reflects a common strategy and it is useful for exploring plateaus. If the offspring is removed in case of equal fitness, or if ties are broken uniformly, steps removing the offspring are idle steps. If the fitness function is flat and ties are broken uniformly, an idle step occurs with probability $1/(\mu+1)$. Idle steps do not affect the equilibrium states of population diversity, but they slow down the process by a factor $\mu/(\mu+1)$, see Remark~\ref{rem:random-tie-breaking}.


\paragraph{Mutation and Crossover Operators:}
One contribution of the paper is to characterise \emph{diversity-neutral} crossover operators, which we will define in Section~\ref{sec:equilibria-crossover}. 
In preparation for this, we define important properties of mutation and crossover operators.


A $k$-ary operator takes $k$ search points $x_1,\ldots,x_k\in \{0,1\}^n$ as input and outputs $y\in\{0,1\}^n$, where $y$ may be random. For example, mutation operators are unary ($1$-ary) operators, and crossover operators are most often binary ($2$-ary), although crossover operators with higher arity exist as well. We will use the notion of \emph{unbiased} operators by Lehre and Witt~\cite{Lehre2010}. Intuitively, an operator is unbiased if it treats bit values and bit positions symmetrically. Formally, we require the following.

\begin{definition}
A $k$-ary operator $\op(x_1,\ldots,x_k)$ is \emph{unbiased} if the following holds for all $x_1,\ldots,x_k \in \{0,1\}^n$.
Let $D(y \mid x_1,\ldots,x_k) := \Pr(\op(x_1,\ldots,x_k) = y)$. 
\begin{itemize}
    \item [(i)] For every permutation of $n$ bit positions $\sigma$ we have
    \[
    D(y \mid x_1,\ldots,x_k) = D(\sigma(y) \mid \sigma(x_1),\ldots,\sigma(x_k)).
    \]
    \item [(ii)] For every $z \in \{0,1\}^n$ we have ($\oplus$ denoting exclusive OR)
    \[
    D(y \mid x_1,\ldots,x_k) = D(y \oplus z \mid x_1 \oplus z,\ldots, x_k \oplus z).
    \]
\end{itemize}
\end{definition}
Most mutation operators are unbiased, including standard bit mutation and the heavy-tailed mutation operators used in \emph{fast EAs/GAs}~\cite{Doerr2017-fastGA}. Many crossover operators are also unbiased, but not all. A detailed discussion by Friedrich et al. can be found in~\cite{Friedrich2022}.

A crossover operator is \emph{respectful}~\cite{Radcliffe1994} if components on which all parents agree are passed on to the offspring, i.e., the output is in the convex hull of the inputs. \added{For our purposes, the following description via masks is useful.}

\begin{definition}
\label{def:respectful}
A $k$-ary operator $\op$ is \emph{respectful} if it chooses a possibly random mask $a \in [k]^n$ (where the probabilities may depend on the parents) such that the $i$-th bit of $y$ is taken from $x_{a_i}$. 
\end{definition}
Note that the mask in Definition~\ref{def:respectful} is not unique in positions in which parents have the same bit. We will consider respectful operators where the mask does not depend on the order of the parents. Here we restrict ourselves to binary operators.

\begin{definition}\label{def:mask-based}
For a binary respectful operator $\cross$, let $M(a, x_1, x_2)$ be the probability of $\cross(x_1, x_2)$ choosing the mask $a \in \{1, 2\}^n$. We call the mask \emph{order-independent} if $M(a, x_1, x_2) = M(a, x_2, x_1)$ for all $x_1,x_2 \in \{0,1\}^n$ and $a \in \{1, 2\}^n$, and we then say for short that $\cross$ has an \emph{order-independent mask} (OIM).
%
\end{definition}
\added{Since a respectful operator can be described by different masks, it can happen that the same respectful operator can either be described by an order-independent mask or by a mask that does depend on the order. For all our results, the \emph{existence} of an order-independent mask is sufficient, so our results also apply to the case described above.}

A respectful operator trivially has an OIM if the mask is created without considering the parents. Uniform crossover,
biased uniform crossover (where each bit is chosen independently from parent $x_1$ with a given probability $c \in [0, 1]$) and $k$-point crossover
are examples of respectful crossovers with OIM. Note that OIM does not imply symmetry between the parents. For instance, the operator which always returns the first parent, that is, $M(\vec{1}, x_1, x_2) = 1$, has an OIM since the mask $\vec 1$ does not depend on the order of the parents; in fact, it does not depend on the parents at all. On the other hand, the bitwise AND operator is respectful, but does not have an OIM, as for bits where both parents differ, the mask has to reflect the unique parent having a bit value of 0. 
We give a formal proof in Lemma~\ref{lem:counterexample-to-respectful-implies-diversity-neutral}.

\paragraph{Diversity Measure:}\label{sec:diversity}
We consider the sum of Hamming distances as a natural and standard~\cite{wineberg2003underlying} diversity measure:
\begin{definition}
For a population $P_t = \{x_1, \dots, x_{\mu}\}$ and a search point $y \in \{0, 1\}^n$ we define \added{
 \[
     S(y) \coloneqq S_{P_t}(y) \coloneqq \sum_{i=1}^{\mu} H(x_i,y)
 \]
 and
 \[
     S(P_t) \coloneqq \sum_{i=1}^\mu S_{P_t}(x_i) = \sum_{i=1}^\mu \sum_{j=1}^\mu H(x_i, x_j).
 \]}
The double sum includes the Hamming distance of each pair $x_i, x_j$ with $i \neq j$ twice. If instead we sum over all $(i,j)$ with $i<j$, we would obtain $S(P_t)/2$. Other rescalings are also interesting. The average value of $S(y)$ with $y\in P_t$ is $S(P_t)/\mu$. The expected Hamming distance of two uniform random points in $P_t$ drawn with replacement is $S(P_t)/\mu^2$, and without replacement it is ${S(P_t)/(\mu^2-\mu)}$. Since those values differ only by a fixed factor from $S(P_t)$, all our results transfer straightforwardly to these other measures.
\end{definition}

The sum of Hamming distances is one of the oldest and most popular diversity metrics~\cite{wineberg2003underlying}. It can be calculated with $O(\mu n)$ operations~\cite{wineberg2003underlying}, which is linear in the input size and hence optimal for all diversity measures that take into account all of a population's genetic information. The idea is simply to count for each bit position~$i$ how many individuals have a 1 at position~$i$. If this number is $c_i$, the contribution to $S(P_t)$ is $c_i (\mu-c_i)$ as this is the number of pairs of population members that have different values at bit~$i$. Consequently, $S(P_t) = \sum_{i=1}^n c_i (\mu-c_i)$. In the context of a \muea, the vector $c_1, \dots, c_n$ and thus $S(P_t)$ can be updated after one generation with $O(n)$ operations, which is again optimal. According to~\cite{Ulrich2012}, the sum of Hamming distances has two desirable properties. Firstly, diversity increases when adding a new search point that is not yet contained in the population (called \emph{monotonicity in species}~\cite{SolowPolasky98}). Secondly, the diversity does not decrease when replacing the population with one where all pairs of solutions have a distance at least as large as the previous one (\emph{monotonicity in distance}~\cite{SolowPolasky98}). 
It does not fulfil the \emph{twinning} property, stating that diversity should remain constant when adding a clone of a search point into the population~\cite{SolowPolasky98}, and it may be maximised by a population forming clusters of search points such that the clusters have a large Hamming distance~\cite{Ulrich2012}.


\paragraph{On stationary distributions and mixing times:} A steady-state \muea or \muga can be described by a Markov chain over the state space of all possible populations. For most mutation and crossover operators, the Markov chain is irreducible if the algorithm runs on a flat fitness landscape. For example, standard bit mutation has a non-zero probability to create any offspring $y$ from any parent $x$. Thus, there is a positive probability of creating any population $P_2$ from any initial population $P_1$ in a sequence of at most $\mu$ generations. The Markov chain is also usually aperiodic since there is a positive probability of adding a clone of the search point being removed, and hence there is a positive self-loop probability. In this case, by the fundamental theorem of Markov chains~\cite[Theorem~6.2]{Motwani1995}, there exists a unique stationary distribution. The expected time to approach the stationary distribution is called \emph{mixing time} and there is a well-established machinery for bounding mixing times (see, e.\,g.~\cite{Levin2008}).  However, the Markov chain lives on the space of all possible populations, which has size $2^{n\mu}$, and even a logarithmic mixing time would be of order $\Omega(n\mu)$. Compared to this, our bound for approaching or crossing the equilibrium state from above is $O(\mu^2\ln^+(n/\mu))$, which can be much smaller. We do not believe that such results can be directly deduced from mixing times.

Of course, the diversity $S(P_t)$ performs a random walk on a much smaller state space. But this is in general not a Markov chain, since there may be very different populations having the same overall diversity, and the possible values of $S(P_{t+1})$ depend on the details of the populations, not only on the value $S(P_t)$.

\section{Drift of Population Diversity for Steady-State EAs without Crossover}
\label{sec:equilibrium-states}

Now we will compute the expected change of $S(P_t)$, i.e., $\E(S(P_{t+1}))$ for a given $P_t$. We break the process down into several steps, and work out 
a unifying formula for a very general situation, see Corollary~\ref{cor:equilibrium-neutral}. This includes the \muea for any unbiased mutation operator (Theorem~\ref{the:equilibrium-mutation-only}), but as we will see later in Section~\ref{sec:equilibria-crossover}, it also includes the \muga with a large variety of crossover operators.

We start with a lemma which describes the expected change for a fixed value of the offspring $y'$.

\begin{lemma}
\label{lem:how-E-S-P-t-plus-1-is-derived}
Consider a population $P_t = \{x_1, \dots, x_\mu\}$ and a search point $y'\in\{0,1\}^n$. Let $P_{t+1} \coloneqq P_t \cup \{y'\} \setminus \{x_d\}$ for a uniformly random $d \in [\mu]$. Then
\[
    \E(S(P_{t+1})) = \left(1 - \frac{2}{\mu}\right) S(P_t) + \frac{2(\mu-1)}{\mu} S_{P_t}(y').
\]
\end{lemma}

\begin{proof}
For any fixed~$d \in [\mu]$, let $P_{t+1}^{-d} \coloneqq P_t \cup \{y'\} \setminus \{x_d\}$. Then
\begin{align}
    & S(P_{t+1}^{-d}) = \sum_{z\in P_{t+1}^{-d}} \sum_{z'\in P_{t+1}^{-d}} H(z, z').\nonumber\\
    \intertext{The double sum contains summands $H(y', x_j)$ for all $j \in [\mu] \setminus \{d\}$ and summands $H(x_i, y')$ for all $i \in [\mu] \setminus \{d\}$ as well as a summand $H(y', y') = 0$. By virtue of $H(x_i, x_j) = H(x_j, x_i)$, this equals}
    & = \sum_{i=1, i \neq d}^{\mu} \sum_{j=1, j \neq d}^{\mu} H(x_i, x_j) + 2\sum_{i=1, i \neq d}^{\mu} H(x_i, y').\nonumber\\
    \intertext{Compared to $S(P_t)$, the double sum is missing summands $H(x_{d}, x_j)$ for all $j \in [\mu] \setminus \{d\}$ and summands $H(x_i, x_{d})$ for all $i \in [\mu] \setminus \{d\}$ as well as a summand $H(x_{d}, x_{d}) = 0$. Thus, this is equal to}
    & = \sum_{i=1}^{\mu} \sum_{j=1}^{\mu} H(x_i, x_j) + 2\sum_{i=1, i \neq d}^{\mu} H(x_i, y') - 2\sum_{i=1}^{\mu} H(x_i, x_d)\nonumber\\
    & = S(P_t) + 2\sum_{i=1, i \neq d}^{\mu} H(x_i, y') - 2\sum_{i=1}^{\mu} H(x_i, x_d)\label{eq:derivation-of-E-S-P-t}.
\end{align}
Owing to the uniform choice of~$d$, we get
\begin{align*}
    & \E(S(P_{t+1})) = \frac{1}{\mu} \sum_{d=1}^{\mu} S(P_{t+1}^{-d})\nonumber\\
    & = S(P_t) + \frac{2}{\mu} \sum_{d=1}^{\mu} \sum_{i=1, i \neq d}^{\mu} H(x_i, y') - \frac{2}{\mu} \sum_{d=1}^{\mu} \sum_{i=1}^{\mu} H(x_i, x_d).\\
    \intertext{The first double sum contains terms $H(x_i, y')$ for all $i \in [\mu]$ exactly $\mu-1$ times. The second double sum equals $S(P_t)$. Thus,}
    & = \left(1 - \frac{2}{\mu}\right) S(P_t) + \frac{2(\mu-1)}{\mu} \sum_{i=1}^{\mu} H(x_i, y'). \qedhere
\end{align*}
\end{proof}

The next lemma tells us how, for given $x,z$, mutating $x$ changes the distance from a fixed search point $z$ in expectation. Interestingly, if the mutation operator is unbiased then the result depends only on the expected number of bit flips, but not on the exact nature of the mutation operator. 

\begin{lemma}\label{lem:unbiased-mutation}
Let $x, z\in \{0,1\}^n$, and let $y$ be created from $x$ by an unbiased mutation operator that flips $\chi$ bits in expectation. Then
\[
    \E(H(z,y)) = \chi +  \left(1 - \frac{2\chi}{n}\right)H(z,x).
\]
\end{lemma}
\begin{proof}
Let $p_i$ be the probability of flipping the $i$-th bit of $x$. By unbiasedness, we have $p_i = p_j$ for all $i,j \in [n]$. We also have $\sum_{i=1}^n p_i = \chi$. Since all $p_i$ are equal, this implies $n p_i = \chi$, or $p_i = \chi/n$.

There are $H(z,x)$ positions on which $x$ and $z$ differ. In expectation, $\chi/n \cdot H(z,x)$ of them are flipped, and each flip decreases the distance from $z$ by one. There are $n-H(z,x)$ positions on which $x$ and $z$ agree. Each such flip \emph{in}creases the distance from $z$ by one, and their expectation
is
$\chi/n \cdot (n-H(z,x))$. Hence, $\E(H(z,y))$ equals
\begin{align*}
H(z,x) - \frac{\chi H(z,x)}{n} + \frac{\chi(n-H(z,x))}{n} 
    =\;& \chi +  \left(1 - 2\chi/n\right)H(z,x). \qedhere
\end{align*}
\end{proof}

Lemmas~\ref{lem:how-E-S-P-t-plus-1-is-derived} and~\ref{lem:unbiased-mutation} together allow us to derive how the diversity evolves if we create the offspring as a mutation of a given string $y$.

\begin{theorem}
\label{the:equilibrium-general}
Consider a population $P_t = \{x_1, \dots, x_\mu\}$ and let $y\in \{0,1\}^n$. Let $y'$ be obtained from $y$ by an unbiased mutation operator which flips $\chi$ bits in expectation. Let $P_{t+1} = P_t \cup \{y'\} \setminus \{x_d\}$ for a uniform random $d \in [\mu]$. Then
\begin{align*}
    &\E(S(P_{t+1})) =\left(1 - \frac{2}{\mu}\right) S(P_t) + 2(\mu-1)\chi
     + \frac{2(\mu-1)}{\mu}\left(1 - \frac{2\chi}{n}\right)  S(y).
\end{align*}
\end{theorem}
 \begin{proof}
     Note that $S(y)= \sum_{i=1}^\mu H(x_i,y)$. 
     Then by Lemma~\ref{lem:how-E-S-P-t-plus-1-is-derived}, the law of total probability and linearity of expectation
 \begin{align}\label{eq:general-drift1}
     \E(S(P_{t+1})) = \left(1 - \frac{2}{\mu}\right) S(P_t) + \frac{2(\mu-1)}{\mu} \E(S(y')).
 \end{align}
 On the other hand, by Lemma~\ref{lem:unbiased-mutation} and again linearity of expectation, for all $i\in [n]$,
 \begin{align}\label{eq:general-drift2}
     \E(H(x_i,y')) = \chi + \left(1 - \frac{2\chi}{n}\right)  H(x_i,y).
 \end{align}
 Summing~\eqref{eq:general-drift2} over all $i$ yields
 \begin{align}\label{eq:general-drift3}
     \E(S(y')) & = \sum_{i=1}^\mu\E(H(x_i,y')) \nonumber\\
     & = \mu\chi + \left(1 - \frac{2\chi}{n}\right)  \sum_{i=1}^\mu H(x_i,y) \nonumber\\
     & = \mu \chi + \left(1 - \frac{2\chi}{n}\right)  S(y).
 \end{align}
 Plugging~\eqref{eq:general-drift3} into~\eqref{eq:general-drift1} yields the theorem.
 \end{proof}

Remarkably, Theorem~\ref{the:equilibrium-general} depends only on $S(y)$, not on $y$ itself. This means that all $y$ with the same value of $S(y)$ yield the same dynamics. 
Moreover, by linearity of expectation the same still holds with $\E(S(y))$ if $y$ is random. The following corollary describes the special case that $\E(S(y)) = S(P_t)/\mu$. As we will see later, this special case covers many interesting situations. In particular, it covers the \muea, where the parent is chosen at random, and it covers the \muga with any unbiased, respectful crossover operator.

\begin{corollary}
\label{cor:equilibrium-neutral}
Consider a population $P_t = \{x_1, \dots, x_\mu\}$. Consider any process that
\begin{enumerate}
    \item creates $y$ by any random procedure such that $\E(S(y)) = S(P_t)/\mu$; 
    \item creates $y'$ from $y$ by an unbiased mutation operator which flips $\chi$ bits in expectation;
    \item sets $P_{t+1} = P_t \cup \{y'\} \setminus \{x_d\}$ for a uniformly random $d \in [\mu]$.
\end{enumerate}
Then
\begin{align*}
    \E(S(P_{t+1})) =\;& \left(1 - \frac{2}{\mu^2} - \frac{4(\mu-1)\chi}{\mu^2 n}\right)S(P_t) + 2 (\mu-1) \chi.
\end{align*}
\end{corollary}
\begin{proof}
    We apply Theorem~\ref{the:equilibrium-general} with a random $y$. By linearity of expectation,
    \begin{align*}
    \E(S(P_{t+1})) & =\left(1 - \frac{2}{\mu}\right) S(P_t) + 2(\mu-1)\chi
     + \frac{2(\mu-1)}{\mu}\left(1 - \frac{2\chi}{n}\right) \E(S(y)) \\
    &  =\left(1 - \frac{2}{\mu}\right) S(P_t) 
     + 2(\mu-1)\chi+\frac{2(\mu-1)}{\mu}\left(1 - \frac{2\chi}{n}\right) \frac{S(P_t)}{\mu} \\
     &  = \left(1 - \frac{2}{\mu}+ \frac{2}{\mu} - \frac{2}{\mu^2} - \frac{4(\mu-1)\chi}{\mu^2 n}\right)S(P_t) 
     + 2(\mu-1)\chi,
\end{align*}
and cancelling the $2/\mu$-terms yields the corollary.
\end{proof}

As an immediate consequence,  
the \muea with any unbiased mutation operator meets the conditions of Corollary~\ref{cor:equilibrium-neutral}. 

\begin{theorem}
\label{the:equilibrium-mutation-only}
Consider any \muea from Algorithm~\ref{alg:steady-state-EA} with any unbiased mutation operator flipping $\chi$ bits in expectation and a population size of~$\mu$ on a flat fitness function. Then for all populations~$P_t$
\begin{align*}
    \E(S(P_{t+1})) =\;& \left(1 - \frac{2}{\mu^2} - \frac{4(\mu-1)\chi}{\mu^2 n}\right)S(P_t) + 2 (\mu-1)\chi.
\end{align*}
\end{theorem}
\begin{proof}
   This is an immediate consequence of Corollary~\ref{cor:equilibrium-neutral}, where $y\in P_t$ is chosen randomly, since such a random parent $y$ satisfies
\[
    \E(S(y)) = \frac{1}{\mu}\sum_{i=1}^\mu \sum_{j=1}^\mu H(x_i, x_j) = \frac{S(P_t)}{\mu}. \qedhere
\]
\end{proof}

\section{Equilibria and Time Bounds}
\label{sec:stabilisation-time}


The preceding results give immediate insights about an equilibrium state for the population diversity. Define 
\begin{align}\label{eq:mu-and-alpha}
\alpha:=2(\mu-1)\chi\quad \text{and}\quad \delta:=\frac{2}{\mu^2}+\frac{4 (\mu-1)\chi}{\mu^2 n},
\end{align}
then Corollary~\ref{cor:equilibrium-neutral} and Theorem~\ref{the:equilibrium-mutation-only} state that
\begin{equation}
\label{eq:drift-with-alpha-and-delta}
    \E(S(P_{t+1})) = (1-\delta)S(P_t) + \alpha.
\end{equation}
This condition was described in~\cite{DoerrNegativeMultiplicativeDrift} as \emph{negative multiplicative drift with an additive disturbance} (in~\cite{DoerrNegativeMultiplicativeDrift} only lower hitting time bounds were given, while we will prove upper bounds). 
An equilibrium state with zero drift is attained for 
\[
    S_0 \coloneqq \frac{\alpha}{\delta} = \frac{(\mu-1)\mu^2 \chi n}{2(\mu-1)\chi + n}
\]
since then
$
    \E(S(P_{t+1}) \mid S(P_t) = S_0) = (1-\delta) \cdot \frac{\alpha}{\delta} + \alpha = \frac{\alpha}{\delta} = S_0
$.

If $(\mu-1)\chi \ll n$ then the equilibrium is close to $(\mu-1)\mu^2\chi$ and the average Hamming distance is $(\mu-1)\chi$, growing linearly in the population size and linearly in the mutation strength~$\chi$.
If $(\mu-1)\chi \gg n$ then the equilibrium is close to $\mu^2 n/2$, that is, the average Hamming distance between population members is roughly $n/2$. This equals the expected Hamming distance between population members in a uniform random population.
Since $2(\mu-1)\chi + n \ge \max\{2(\mu-1)\chi, n\}$, the average Hamming distance is at most
\[
    \alpha/(\delta \mu^2)  \le \min\left\{(\mu-1)\chi, n/2\right\}
\]
hence bounded by the value $n/2$ for a uniform random population.


We stress again that for given $\mu$ and $n$, the equilibrium value $\alpha/\delta$ only depends on the expected number $\chi$ of flipped bits. For example, both RLS mutation, which flips exactly one bit, and standard bit mutation with mutation rate $1/n$ have the same value $\chi =1$ and hence the same equilibrium state. Recently, another type of mutation operator has become quite popular, where the probability $p_k$ of flipping $k$ bits has a \emph{heavy tail}~\cite{Doerr2017-fastGA}. Usually, it scales as $p_k \sim k^{-\tau}$ for some $\tau >1$. In many applications, all values of $\tau$ lead to similar results. However, here they lead to qualitatively different behaviour due to different values of $\chi$. More precisely, $\tau > 2$ leads to $\chi = \Theta(1)$~\cite{newman2005power}, which gives the same dynamics as standard bit mutation with slightly different mutation rate $\Theta(1/n)$. In particular, $\alpha/\delta = \Theta(\mu^3)$ for $\mu \le n$. On the other hand, $\tau \in (1,2)$ leads to 
$\chi = \Theta( \sum_{k=1}^{n} k\cdot p_k) = \Theta(\sum_{k=1}^n k^{1-\tau}) = \Theta(n^{2-\tau})$.
Thus for $\mu \le n^{\tau-1}$ the equilibrium state is $\alpha/\delta = \Theta(n\mu^2)$. For constant $\mu$, this means that the equilibrium state jumps from $\Theta(1)$ to $\Theta(n)$ as $\tau$ crosses the threshold $\tau =2$. For $\tau =2$, we get an intermediate regime of $\chi = \Theta(\log n)$~\cite{newman2005power}.

For another perspective on the equilibrium state we consider the distance $D(P_t) := S(P_t)-\alpha/\delta$. With \eqref{eq:drift-with-alpha-and-delta} this changes as 
\begin{align*}
\label{eq:drift-with-alpha-and-delta}
    \E(D(P_{t+1})) &= \E(S(P_{t+1})) -\alpha/\delta \\
    & = (1-\delta)S(P_t) + \alpha -\alpha/\delta 
    = (1-\delta)D(P_t).
\end{align*}
Hence, the distance from the equilibrium state shows a multiplicative drift. However, 
note crucially that $D(P_t)$ may take \emph{positive and negative} values and the multiplicative drift theorem~\cite{Doerr2012a} is not applicable. 
The process is quite different from the usual situation of multiplicative drift, in which the target state is reached quickly. In fact, the equilibrium state $S(P_t) =\alpha/\delta$ may never be reached, since it might not be achievable due to rounding issues or if the diversity changes in large steps. However, we will show that the diversity will quickly reach an \emph{approximation} of the equilibrium state, or that the equilibrium state will be overshot.

The following theorem gives two upper time bounds. When starting with a diversity of $S(P_t) > (1+\varepsilon)\alpha/\delta$, we bound the expected time to reach a diversity at most $(1+\varepsilon)\alpha/\delta$. This reflects a scenario where a population has an above-average diversity and we ask how long it takes for  diversity to reduce. Similarly, starting with a population of little diversity, $S(P_t) < (1-\varepsilon)\alpha/\delta$, we estimate the expected time for diversity to increase to at least $(1-\varepsilon)\alpha/\delta$. 
As it might be of independent interest, we formulate this theorem for general finite stochastic processes $(X_t)_{t \geq 0}$ in $\mathbb{N}_0$ whose drift is bounded from above or below by $(1-\delta)X_t + \alpha$, respectively. 

\begin{theorem}
\label{thm:borders}
Fix $0 < \varepsilon \leq 1$. Let $(X_t)_{t \geq 0}$ with $X_t \in \{0, \ldots,X_{\max}\}$ be a stochastic process. Let $T_{\varepsilon,\downarrow} := \inf\left\{t \mid X_t \leq (1+\varepsilon)\frac{\alpha}{\delta} \right\}$ and $T_{\varepsilon,\uparrow} := \inf\left\{t \mid X_t \geq (1-\varepsilon)\frac{\alpha}{\delta} \right\}$.
\begin{itemize}[leftmargin=1em,itemindent=1em,topsep=0.5ex]
\item[(i)]
If $\E(X_{t+1} \mid X_t = x) \le (1-\delta) x + \alpha$ for all $x > \frac{\alpha}{\delta} (1+\varepsilon)$ then 
\begin{align*}
\E(T_{\varepsilon,\downarrow}) \le \frac{4}{\varepsilon \delta} \ln\Big(\frac{2\delta X_{\max}}{\varepsilon \alpha}\Big).
\end{align*}
\item[(ii)]
If $\E(X_{t+1} \mid X_t = x) \ge (1-\delta) x + \alpha$ for all $x < (1-\eps)\frac{\alpha}{\delta}  $ and there is a $\Delta_{\max} \in \mathbb{R}$ such that $|X_{t+1} - X_t| \le \Delta_{\max}$ for all~$t$ then
\begin{align*}
\E(T_{\varepsilon,\uparrow}) \le \frac{4\Delta_{\max}}{\varepsilon \alpha} \ln\Big(\frac{ 2 \alpha + 2\delta \Delta_{\max}}{\eps\alpha}\Big).
\end{align*}
\end{itemize}
\end{theorem} 


\begin{proof}
%
We use a direct argument, 
similar to the proof of the tail bound for multiplicative drift~\cite{DoerrG13}.

(i): 
We may safely assume $\E(X_{t+1} \mid X_t = x) \le (1-\delta) x + \alpha$ for \emph{all} $x$
as for $x \le \alpha/\delta \cdot (1+\varepsilon)$ we are done. We show by induction that
\begin{equation}
\label{eq:expanding-Xt}
    \E(X_t \mid X_0) \leq \sum_{i=0}^{t-1} (1-\delta)^i \alpha + (1-\delta)^t X_0.
\end{equation}
For the base case $t=0$ we have $\E(X_0 \mid X_0) = X_0$ and $\sum_{i=0}^{t-1} (1-\delta)^i \alpha + X_0 = X_0$ as the sum is empty. Now assume the claim holds for $\E(X_t \mid X_0)$. Using the law of total expectation $\E(\E(X \mid Y)) = \E(X)$
\begin{align*}
    \E(X_{t+1} \mid X_0) =\;& \E(\E(X_{t+1} \mid X_t) \mid X_0)\\
    \leq \;& \E((1-\delta)X_t + \alpha \mid X_0)\\
    =\;& (1-\delta) \E(X_t \mid X_0) + \alpha.
\end{align*}
Applying the induction hypothesis, we get
\begin{align*}
\E(X_{t+1} \mid X_0) 
    \leq \;& (1-\delta)\left(\sum_{i=0}^{t-1} (1-\delta)^i \alpha + (1-\delta)^t X_0\right) + \alpha\\
    =\;& \sum_{i=0}^{t-1} (1-\delta)^{i+1} \alpha + (1-\delta)^{t+1} X_0 + \alpha\\
    =\;& \sum_{i=0}^{t} (1-\delta)^{i} \alpha + (1-\delta)^{t+1} X_0.
\end{align*}
From~\eqref{eq:expanding-Xt}, we get, bounding the sum by an infinite series $\sum_{i=0}^\infty (1-\delta)^i = \frac{1}{\delta}$ and using $1-\delta \le e^{-\delta}$ as well as $X_0 \le X_{\max}$,
\[
    \E(X_t \mid X_0) \leq \sum_{i=0}^{t-1} (1-\delta)^i \alpha + (1-\delta)^t X_0 \le \frac{\alpha}{\delta} + e^{-\delta t} \cdot X_{\max}.
\]
Choosing $t := \ln(X_{\max} \cdot \delta/\alpha \cdot 2/\varepsilon)/\delta$, we obtain
\[
    \E(X_t \mid X_0) \le \frac{\alpha}{\delta} + \frac{1}{X_{\max}} \cdot \frac{\alpha}{\delta} \cdot \frac{\varepsilon}{2} \cdot X_{\max} =
    \frac{\alpha}{\delta} \cdot \left(1 + \frac{\varepsilon}{2}\right).
\]
By Markov's inequality we get, for all values of~$X_0$,
\[
    \Prob\left(X_t \ge \frac{\alpha}{\delta} \cdot (1+\varepsilon)\right) \le \frac{\frac{\alpha}{\delta} \cdot \left(1 + \frac{\varepsilon}{2}\right)}{\frac{\alpha}{\delta} \cdot \left(1 + \varepsilon\right)} = \frac{1 + \varepsilon/2}{1+\varepsilon}
\]
and thus
$
    \Prob\left(X_t < \frac{\alpha}{\delta} \cdot (1+\varepsilon)\right) \ge 1 - \frac{1 + \varepsilon/2}{1+\varepsilon} = \frac{\varepsilon/2}{1+\varepsilon}$.

In case $X_t > \frac{\alpha}{\delta} \cdot (1+\varepsilon)$ we repeat the above arguments with a further phase of $t$ steps. (Here we exploit that the above bound was made independent of~$X_0$.) The expected number of phases required is at most $\frac{1+\varepsilon}{\varepsilon/2} \le 4/\eps$ as $\eps \le 1$. This gives an upper bound of $4t/\varepsilon$.

\vspace{0.2cm}
(ii): \added{Since we are only interested in $T_{\varepsilon,\uparrow}$, we may assume that the process becomes stationary afterwards, i.e.\! $X_{T_{\varepsilon,\uparrow}} = X_{T_{\varepsilon,\uparrow}+1} = X_{T_{\varepsilon,\uparrow}+2} = \ldots$. Moreover, we may assume $X_0 <(1-\eps)\alpha/\delta$, since otherwise there is nothing to show.} Define $Y_{\max} \coloneqq \min\{\alpha/\delta + \Delta_{\max}, X_{\max}\}$ and $Y_t:=Y_{\max}-X_t$. 
We first show that $0 \le Y_t \le X_{\max}$ for all $t \ge 0$. If $Y_{\max} = X_{\max}$ this is obvious. 
\added{For the case $Y_{\max} = \alpha/\delta + \Delta_{\max}$ it suffices to show $0 \le Y_t \le X_{\max}$ for all $t \le T_{\varepsilon,\uparrow}$ since the process is stationary afterwards. For $t=0$ the bound holds by assumption.} 
For $0 < t \le T_{\varepsilon, \uparrow}$ we have $X_{t-1} \le \alpha/\delta$ and $|X_t - X_{t-1}| \le \Delta_{\max}$ by assumption, hence $X_t \le \alpha/\delta + \Delta_{\max} = Y_{\max}$ and $Y_t = Y_{\max} - X_t \ge 0$. 

Let $\eps' := \frac{\eps \alpha}{\delta Y_{\max}-\alpha},  \delta' := \delta, \alpha' := \delta Y_{\max}-\alpha$.
Then the event ``$X_t \ge (1-\eps)\tfrac{\alpha}{\delta}$'' is equivalent to the event ``$Y_t \le (1+\eps')\tfrac{\alpha'}{\delta'}$'', because
\begin{align*}
    (1+\eps')\tfrac{\alpha'}{\delta'} & = (1 + \eps') (Y_{\max}-\tfrac{\alpha}{\delta}) 
    = Y_{\max} - \tfrac{\alpha}{\delta} + \eps' \tfrac{\delta Y_{\max}-\alpha}{\delta} \\
    & = Y_{\max} - \tfrac{\alpha}{\delta} + \eps \tfrac{\alpha}{\delta}
    = Y_{\max} - (1-\eps)\tfrac{\alpha}{\delta},
\end{align*}
and because $Y_t = Y_{\max}- X_t$. We can describe $T_{\varepsilon,\uparrow}$ as the first point in time when $Y_t \le (1+\eps')\tfrac{\alpha'}{\delta'}$, since this is equivalent to $X_t \ge (1-\eps)\tfrac{\alpha}{\delta}$.

Moreover, the same calculation shows that for all $y> (1+\eps')\tfrac{\alpha'}{\delta'}$ the event $``Y_t = y$'' implies $X_t < (1-\eps)\tfrac{\alpha}{\delta}$, so that the drift bound for $X_t$ is applicable. Hence, for any such $y$, the drift of $Y_t$ is 
\begin{align*}
\E(Y_{t+1} \mid Y_t = y) \;&= 
\E \left(Y_{\max} - X_{t+1} \mid X_t = Y_{\max} - y \right) \\
\;&= Y_{\max} - \E \left(X_{t+1} \mid X_t = Y_{\max} - y \right)\\
\;& \leq Y_{\max} - (1-\delta)\left(Y_{\max} - y\right) - \alpha 
= (1-\delta')y+\alpha'. 
\end{align*}
Therefore, the prerequisites of part (i) are satisfied by $Y_t$ with parameters $\eps',\delta'$ and $\alpha'$. Hence, part (i) applied to $Y_t$ gives
\begin{align*}
\E(T_{\varepsilon,\uparrow}) &  \leq \frac{4}{\eps'\delta'} \cdot \ln\left(\frac{2\delta' Y_{\max}}{\eps' \alpha'}\right) 
\le \frac{4\Delta_{\max}}{\eps\alpha} \ln\left(\frac{ 2 \alpha + 2\delta \Delta_{\max}}{\eps\alpha}\right).
\qedhere
\end{align*}
\end{proof}

\added{To apply Theorem~\ref{thm:borders} to our situation, we first prove a bound on $\Delta_{\max}$.}

\begin{lemma}
\label{lem:hamming}
    Consider a population $P_t=\{x_1, \ldots,x_\mu\}$. Consider any process that creates $y$ by any random procedure and sets $P_{t+1}=P_t \cup \{y\} \setminus \{x_d\}$ for some $d \in [\mu]$. Then $\vert{S(P_{t+1}) - S(P_t)}\vert \leq 2(\mu-1)n$.
\end{lemma}
\begin{proof}
\added{By Equation~\eqref{eq:derivation-of-E-S-P-t} we have 
\[
S(P_{t+1})-S(P_t) = 2\sum_{i=1, i \neq d}^{\mu} H(x_i, y') - 2\sum_{i=1,i\neq d}^{\mu} H(x_i, x_d),
\]
and the bound follows since both the positive and the negative term are at most $2(\mu-1)n$.}
%
\end{proof}

\begin{theorem}
\label{thm:borders-specific}
Consider a steady-state evolutionary algorithms meeting the conditions of Corollary~\ref{cor:equilibrium-neutral} with $\alpha \coloneqq 2(\mu-1)\chi$ and $\delta \coloneqq \frac{2}{\mu^2} + \frac{4(\mu-1)\chi}{\mu^2 n}$ as in~\eqref{eq:mu-and-alpha}. Fix $0 < \varepsilon \leq 1$ and let $X_t \coloneqq S(P_t)$. Let $T_{\varepsilon,\downarrow} := \inf\left\{t \mid X_t \leq (1+\varepsilon)\frac{\alpha}{\delta} \right\}$ and $T_{\varepsilon,\uparrow} := \inf\left\{t \mid X_t \geq (1-\varepsilon)\frac{\alpha}{\delta} \right\}$.
Then
\begin{align}\label{thm:eq:specific1}
\E(T_{\varepsilon,\downarrow}) 
=\;&  O\left(\frac{\mu \cdot \min\{\mu,n/\chi\}}{\eps} \cdot \ln\left(\frac{1+n/(\mu\chi)}{\eps}\right)\right),\\
\E(T_{\varepsilon,\uparrow}) =\;& O \left(\frac{n}{\eps \chi} \cdot \ln\left(\frac{1+n/(\mu^2\chi)}{\eps}\right)\right).\label{thm:eq:specific2}
\end{align}
\end{theorem}

\begin{proof}
In order to apply Theorem~\ref{thm:borders} to our case of $(X_t)_{t \geq 0} = (S(P_t))_{t \geq 0}$, we may set $\Delta_{\max} := 2(\mu-1)n$ by Lemma~\ref{lem:hamming}. Moreover, we have $S(P_t)_{\max} \leq \mu^2 n$, since two individuals have Hamming distance at most $n$ and so the diversity is at most $2 \binom{\mu}{2} n$. \added{For~\eqref{thm:eq:specific1}, we have $1/\delta = \frac{\mu^2 n}{2n + 4(\mu-1) \chi}$, which implies $1/\delta \in \Theta(\mu \cdot \min\{\mu,n/\chi\})$. Now~\eqref{thm:eq:specific1} follows immediately by plugging this into the bounds from Theorem~\ref{thm:borders}. For~\eqref{thm:eq:specific2}, note that $\Delta_{\max}/\alpha = n/\chi$. Thus, Theorem~\ref{thm:borders} implies
\begin{align*}
\E(T_{\varepsilon,\uparrow}) \le \frac{4n}{\eps\chi} \ln\left(\frac{2}{\eps} + \frac{2n}{\eps\chi}\cdot \left(\frac{2}{\mu^2}+\frac{4(\mu-1)\chi}{\mu^2n}\right) \right)
= O \left(\frac{n}{\eps \chi} \cdot \ln\left(\frac{1+n/(\mu^2\chi)}{\eps}\right)\right),
\end{align*}
where we could omit the last term in the logarithm since $(\mu-1)/\mu^2 = O(1)$.
}
\end{proof}

If $\eps$ is constant, $\chi = \Theta(1)$ and $\mu \le n$, then the bounds further simplify to
\[
\E(T_{\varepsilon,\downarrow}) \in O(\mu^2\ln^+(n/\mu)) \quad \text{ and } \quad\E(T_{\varepsilon,\uparrow}) \in O(n \ln^+(n/\mu^2)).
\]
Note that the bound for $\E(T_{\varepsilon,\downarrow})$ depends very mildly (or not at all) on $n$, so the speed of reducing diversity is almost unaffected by the problem dimension. 
The bound on $\E(T_{\varepsilon,\uparrow})$ applies to a monomorphic population where there is initially no diversity. There are settings in which $\E(T_{\varepsilon,\uparrow}) = \Omega(n)$ for $\mu=2,\chi=1$ (start with two clones and, with probability $1/n$, flip all $n$ bits), thus there are processes in which it is harder to create diversity than to reduce it.

We remark that there are also overshoot-aware multiplicative drift theorems~\citep{BuzdalovDDV22} which could also be directly applied in the situation of Theorem~\ref{thm:borders}, but those leads to poor results since the upper bounds include the expected overshoot, which may be very large.

Note that Theorem \ref{thm:borders} only estimates the expected time to pass the borders $(1+\varepsilon)\frac{\alpha}{\delta}$ and 
$(1-\varepsilon)\frac{\alpha}{\delta}$, respectively. 
It does not guarantee that the diversity hits the interval $[1-\varepsilon,1+\varepsilon]\frac{\alpha}{\delta}$. 

\begin{definition}
    Given a positive constant $\varepsilon>0$ and an initial population $P_t$ we define the first time $T_\varepsilon$ when the diversity $S(P_t)$ is in the equilibrium as
    \[
    T_\varepsilon := \inf \left\{t \mid S(P_t) \in [(1-\varepsilon)\tfrac{\alpha}{\delta},(1+\varepsilon)\tfrac{\alpha}{\delta}]\right\}.
    \]
\end{definition}

 \added{In general, $T_\eps$ does not need to be finite.} We will give a sufficient condition for not skipping over the interval of states close to the equilibrium. The key is that the diversity can change at most by $2(\mu-1)n$ in the setting of Corollary~\ref{cor:equilibrium-neutral} 
 and Theorem~\ref{the:equilibrium-mutation-only}.

\begin{corollary}
\label{cor:equilibrium}
    If $\varepsilon \mu^2 \chi \geq n+2(\mu-1)\chi$ (for example for $\mu \in \Theta(n), \varepsilon \in \Theta(1)$ and $\chi=1$) then $T_\varepsilon=T_{\varepsilon,\downarrow}$ if $S(P_0)>(1+\varepsilon)\frac{\alpha}{\delta}$ and $T_\varepsilon=T_{\varepsilon,\uparrow}$ if $S(P_0)<(1-\varepsilon)\frac{\alpha}{\delta}$, respectively.
\end{corollary}

\begin{proof}
Suppose that $S(P_0)>(1+\varepsilon)\frac{\alpha}{\delta}$. Let $t:=T_{\varepsilon,\downarrow}-1$. Then we obtain $S(P_{t+1}) \leq (1+\varepsilon)\frac{\alpha}{\delta}$. Moreover,
\begin{align*}
S(P_t)-(1-\varepsilon)\frac{\alpha}{\delta} \;&> (1+\varepsilon)\frac{\alpha}{\delta}-(1-\varepsilon)\frac{\alpha}{\delta} \\
\;&=  \frac{4\varepsilon \mu^2 \chi \cdot (\mu-1)n}{2n + 4(\mu-1) \chi}\\
\;&\geq \frac{(4n+8(\mu-1)\chi) \cdot (\mu-1)n}{2n+4(\mu-1)\chi}\\
\;& = 2(\mu-1)n.
\end{align*}
Since $S(P_t)-S(P_{t+1}) \leq 2(\mu-1) n$ by Lemma~\ref{lem:hamming}, we obtain $S(P_{t+1}) \in [1-\varepsilon,1+\varepsilon]\frac{\alpha}{\delta}$. 

Suppose $S(P_0)<(1+\varepsilon)\frac{\alpha}{\delta}$. Let $t:=T_{\varepsilon,\uparrow}-1$. Then we obtain $S(P_{t+1}) \geq (1-\varepsilon)\frac{\alpha}{\delta}$ and 
\[
(1+\varepsilon)\frac{\alpha}{\delta} - S(P_t) > (1+\varepsilon)\frac{\alpha}{\delta} - (1-\varepsilon)\frac{\alpha}{\delta} \geq 2(\mu-1)n. \qedhere
\]
\end{proof}

\added{In general, without restriction such as in Corollary~\ref{cor:equilibrium},  it is possible that the process never comes close to the equilibrium.} 
The simplest (artificial) example is the following. Suppose  $\mu=2$ (so $\mu \in o(\sqrt{n})$), $\varepsilon=\frac{1}{3}$, and $\chi=n$ (i.e. every bit is flipped with probability $1$), we omit crossover and the population initialises with two clones. Then we have $S(P_t) \in \{0,2n\}$ for every $t$ and
\[
[1-\varepsilon,1+\varepsilon] \tfrac{\alpha}{\delta} = \tfrac{4}{3}n[\tfrac{2}{3},\tfrac{4}{3}] = [\tfrac{8}{9}n,\tfrac{16}{9}n].
\]
Therefore $T_\varepsilon=\infty$, but $T_{\varepsilon,\uparrow} \leq 1$ and $T_{\varepsilon,\downarrow} \leq 1$.

\section{Steady-State GA with Crossover}
\label{sec:equilibria-crossover}

Now we turn to steady-state GAs that perform crossover before applying mutation to the resulting offspring (see Algorithm~\ref{alg:steady-state-GA}). Quite surprisingly, for nearly all common crossover operators, including crossover does not change the diversity equilibrium.

A sufficient condition is the following, which we term \emph{diversity-neutral}, as the diversity equilibrium does not change when applying such a crossover operator.
\begin{definition}
\label{def:diversity-neutral}
We call a crossover operator $\cross$ \emph{diversity-neutral} if it has the following property. For all $x_1,x_2,z \in \{0, 1\}^n$,
\begin{equation}
    \label{eq:crossover-property}
\E(H(\cross(x_1, x_2), z) + H(\cross(x_2, x_1), z)) = H(x_1, z) + H(x_2, z).
\end{equation}
\end{definition}

We shall see in Section~\ref{sec:crossover-operators} that common crossover operators like uniform crossover and $k$-point crossover are diversity-neutral. 

We will show that the \muga with any diversity-neutral crossover operator meets the conditions of Corollary~\ref{cor:equilibrium-neutral}. Hence, we obtain the following theorem.

\begin{theorem}
\label{the:equilibrium-crossover}
Consider the \muga with any \wellbehaved crossover operator~$c$, any unbiased mutation operator flipping $\chi$ bits in expectation and a population size of~$\mu$ on a flat fitness function. Then for all populations $P_t$,
\begin{align}\label{eq:the:equilibrium-crossover}
    \E(S(P_{t+1})) =\; \added{(1-\delta)S(P_t)+\alpha =\;}& \left(1 - \frac{2}{\mu^2} - \frac{4(\mu-1)\chi}{\mu^2 n}\right)S(P_t) + 2\chi (\mu-1),
\end{align}
\added{where $\delta,\alpha$ are as in~\eqref{eq:mu-and-alpha}.}
\end{theorem}
\begin{proof}
Let $y=\cross(x_i,x_j)$ and $y' = \cross(x_j,x_i)$, where $x_i$ and $x_j$ are the randomly selected parents. We only need to show that $\E(S(y)) = S(P_t)/\mu$, then the theorem follows from Corollary~\ref{cor:equilibrium-neutral}. By definition of diversity neutral crossover, we have for all $k\in [\mu]$, 
\begin{align*}
    \E(H(y,x_k) + H(y',x_k)) = H(x_i,x_k) + H(x_j,x_k).
\end{align*}
Summing over all $k$ yields $S(y)+S(y')$ inside the expectation on the left hand side, and $S(x_i)+S(x_j)$ on the right hand side. Therefore, 
$
    \E(S(y)+ S(y')) = S(x_i) + S(x_j)$.
Now we use that $x_i$ and $x_j$ are chosen uniformly at random. Hence,
\begin{align}\label{eq:one-step-drift-1}
    \E(S(y)+ S(y')) = \frac{1}{\mu^2}\sum_{i=1}^\mu \sum_{j=1}^\mu (S(x_i) + S(x_j)) = \frac{2S(P_t)}{\mu}.
\end{align}

By the symmetric choice of the parents $x_i$ and $x_j$, $\E(S(y)) = \E(S(y'))$, and thus $\E(S(y)) = \tfrac12(\E(S(y)+ S(y'))) = S(P_t)/\mu$. 
\end{proof}

\begin{remark}\label{rem:no-replacement}
Theorem~\ref{the:equilibrium-crossover} still holds if we choose the parents without replacement. 
\end{remark}
\begin{proof}
     We show that~\eqref{eq:one-step-drift-1} still holds in this case. The rest of the proof carries over.
     In order to choose the parents without replacement, we can first take $x_j$ uniformly at random and $x_i$ can then be every individual except $x_j$. So we obtain
  \begin{align*}
     \E(S(y)+ S(y')) \;& = \frac{1}{(\mu-1)\mu}\sum_{i=1}^\mu \sum_{j=1, j \neq i}^\mu (S(x_i) + S(x_j)) \\
    \;& = \frac{S(P_t)}{\mu} + \frac{1}{(\mu-1)\mu} \sum_{i=1}^\mu \sum_{j=1, j \neq i}^\mu S(x_j) \\
    \;& = \frac{S(P_t)}{\mu} + \frac{(\mu-1)S(P_t)}{(\mu-1)\mu} = \frac{2S(P_t)}{\mu}.
  \end{align*}
     The third equality holds, because we sum up $S(x_j)$ exactly $(\mu-1)$ times for every $j \in [\mu]$. So indeed~\eqref{eq:one-step-drift-1} still holds.
\end{proof}

\added{We assumed in Algorithms~\ref{alg:steady-state-EA} and~\ref{alg:steady-state-GA} that they break ties in favour of the offspring. In flat landscapes this means that the offspring is never discarded. We now transfer our results to variants in which the algorithm may also discard the offspring.
\begin{remark}\label{rem:random-tie-breaking}
If the \muga does not favour the offspring over parents but instead breaks ties uniformly at random, then the conclusion of Theorem~\ref{the:equilibrium-crossover} still holds with~\eqref{eq:the:equilibrium-crossover} replaced by
\begin{align*}
    \E(S(P_{t+1})) =(1 - \tilde\delta )S(P_t) + \tilde\alpha \quad \text{ with } \quad \tilde\delta := \tfrac{\mu}{\mu+1}\delta,\ \tilde\alpha := \tfrac{\mu}{\mu+1}\alpha,
\end{align*}
where $\delta,\alpha$ are as in~\eqref{eq:mu-and-alpha}. In particular, the process has the same equilibrium state $\tilde \alpha/\tilde\delta = \alpha/\delta$ and Theorem~\ref{thm:borders} still holds with $\tilde \delta$ and $\tilde\alpha$ instead of $\delta$ and $\alpha$. Note that the bounds in Theorem~\ref{thm:borders-specific} are increased precisely by a factor $(\mu+1)/\mu$ since the additional factors in the logarithms cancel out. Since $(\mu+1)/\mu = \Theta(1)$, Theorem~\ref{thm:borders-specific} still holds unchanged.
\end{remark}
\begin{proof}
 Let $A_t$ denote the event that the individual which we remove is not the offspring. Note that our results obtained so far always assumed $A_t$. By the law of the total probability,
 \begin{align*}
  \E(S(P_{t+1})) \;&=P(A_t) \cdot \E(S(P_{t+1})\mid A_t) +  P(\bar{A_t}) \cdot \E(S(P_{t+1})\mid \bar{A_t}) \\
  \;&= \frac{\mu}{\mu+1}\E(S(P_{t+1})) + \frac{1}{\mu+1}S(P_t).
  \end{align*}
  Therefore, by~\eqref{eq:the:equilibrium-crossover},
  \begin{align*}
 \E(S(P_{t+1})) \;&= \frac{\mu}{\mu+1}\left(1 - \delta\right)S(P_t) + \frac{\mu}{\mu+1}\alpha + \frac{1}{\mu+1}S(P_t)\\
 \;&= \left(1 - \frac{\mu}{\mu+1}\delta\right)S(P_t) + \frac{\mu}{\mu+1}\alpha. \qedhere
 \end{align*}
 \end{proof}
}
\section{Classifying Diversity-Neutral Crossover Operators}
\label{sec:crossover-operators}


In this section we classify several known crossover operators into \wellbehaved ones and those that are not \wellbehaved. 
\subsection{Structural Results}

We start with structural results that connect \wellbehaved with the properties unbiased, respectful \added{and having order-independent mask (OIM)}, see Section~\ref{sec:preliminaries}. We will show that for unbiased crossover operators, \wellbehaved is equivalent to respectful. However, outside the class of unbiased operators, this is not true. While every \wellbehaved operator is still respectful, we show that the converse is false in general, but holds for the very large class of respectful operators with OIM. 

%



\begin{lemma}
\label{lem:property-one-implies-respectful}
Every \wellbehaved crossover operator is respectful.
\end{lemma}

\begin{proof}
 Let $x_1,x_2$ be parents that both have a one-bit in position~$i$. Let $\cross$ be a \wellbehaved crossover operator and $\mathcal E$ be the event that the offspring $\cross(x_1,x_2)$ has a zero-bit in position~$i$. We will assume $\Pr(\mathcal E) >0$ and derive a contradiction. The case that both parents have a zero-bit in position~$i$ is handled similarly. 
Suppose that the event $\mathcal{E}$ appears.
Let $z_0$ and $z_1$ be two search points which are identical in all positions except for position~$i$, where $z_0$ has a zero-bit and $z_1$ has a one-bit at position~$i$. Then
\begin{align}\label{eq:proof-of-respectful-1}
    H(x_1,z_0) = H(x_1,z_1) + 1 \quad \text{and} \quad H(x_2,z_0) = H(x_2,z_1) + 1.
 \end{align}
 Moreover, since $z_0$ and $z_1$ differ in exactly one position,
 $H(y,z_0) - H(y,z_1) \in \{-1,1\}$ for all $y\in \{0,1\}^n$. In particular, $H(y,z_0) - H(y,z_1) \le 1$, and it is a strict inequality if and only if $y$ has a zero-bit in position~$i$. For $y=\cross(x_1,x_2)$, this implies
 \begin{align}\label{eq:proof-of-respectful-2}
     \E(H(\cross(x_1,x_2),z_0) - H(\cross(x_1,x_2),z_1)) 
     =\;& \Prob(\mathcal{E}) \cdot (-1) + \Prob(\overline{\mathcal{E}}) \cdot 1 < 1,
 \end{align}
 where the inequality is strict because we have assumed $\Prob(\mathcal{E}) > 0$.
 For $y=\cross(x_2,x_1)$ we obtain
 \begin{align}\label{eq:proof-of-respectful-3}
     \E(H(\cross(x_2,x_1),z_0) - H(\cross(x_2,x_1),z_1)) \le 1,
 \end{align}
 where this time we cannot claim a strict inequality since we have not made any assumption on $\cross(x_2,x_1)$. Adding up \eqref{eq:proof-of-respectful-2} and \eqref{eq:proof-of-respectful-3}, 
 we obtain
 \begin{align*}
    \E(H(\cross(x_1,x_2),z_0) +  H(\cross(x_2,x_1),z_0) 
     & - H(\cross(x_1,x_2),z_1)
     - H(\cross(x_2,x_1),z_1)) < 2.
 \end{align*}
 But the left hand side equals
 \begin{align*}
     H(x_1,z_0) +  H(x_2,z_0)   - H(x_1,z_1)
     - H(x_2,z_1)) \stackrel{\eqref{eq:proof-of-respectful-1}}{=} 2,
 \end{align*}
which is a contradiction to~\eqref{eq:crossover-property}. Hence, the assumption $\Pr(\mathcal E) >0$ must have been false, and therefore the offspring of $x_1$ and $x_2$ must have a one-bit in position~$i$ with probability $1$. 
This concludes the proof.
\end{proof}

Next, we show that the converse is not true.
\begin{lemma}
\label{lem:counterexample-to-respectful-implies-diversity-neutral}
Not every respectful crossover is \wellbehaved.
\end{lemma}
\begin{proof}

For $x_1,x_2 \in \{0,1\}^n$ we define $\cross(x_1,x_2)$ as the bit-wise AND of $x_1$ and $x_2$. The operator is respectful since $1$ AND $1$ is 1 and 0 AND 0 is 0. 

Now for any two search points $x_1,x_2 \in \{0,1\}^n$ with $x_2 = \overline{x_1}$ and $z = \vec 0$, we have
$H(x_1, z) + H(x_2, z) = n$ as every bit is set to 1 in exactly one parent. However, $\cross(x_1, x_2) = \vec 0$ and so
\[
    \E(H(\cross(x_1,x_2),z) +  H(\cross(x_2,x_1),z)) = 0 \neq H(x_1,z) +  H(x_2,z).
\]
So this crossover is not \wellbehaved.
\end{proof}

The counterexample from Lemma~\ref{lem:counterexample-to-respectful-implies-diversity-neutral} has a strong bias towards setting bits to~0. It is thus not unbiased. 
Now we show that adding OIM gives a sufficient condition to be \wellbehaved. \added{Note that this implies that the AND operator used in the proof of Lemma~\ref{lem:counterexample-to-respectful-implies-diversity-neutral} does not have OIM.}

\begin{lemma}
\label{lem:mask-based-crossovers}
All respectful crossovers with OIM are \wellbehaved. 
\end{lemma}

\begin{proof}
We show for all $z \in \{0, 1\}^n$ and for each bit~$i$ that
\[
    \E(|\cross(x_1, x_2)_i - z_i| + |\cross(x_2, x_1)_i - z_i|) = |(x_1)_i-z_i| + |(x_2)_i - z_i|.
\]
Taking the sum over all $i \in [n]$ turns all absolute differences of bits $|a_i-b_i|$ in the above expression into Hamming distances $H(a, b)$, yielding~\eqref{eq:crossover-property}.
If $(x_1)_i=(x_2)_i$ then the equation is immediate since the left hand side simplifies to $\E(|(x_1)_i-z_i| + |(x_2)_i-z_i|)$ (since $\cross$ is respectful) and the expression is deterministic.

If $(x_1)_i=1-(x_2)_i$ then $\cross$ with OIM implies
\[
    \Prob(\cross(x_1, x_2)_i=(x_1)_i) = \Prob(\cross(x_2, x_1)_i=(x_2)_i) \eqqcolon p.
\]
With probability $q \coloneqq 1-p$, $\cross(x_1, x_2)_i=1-(x_1)_i=(x_2)_i$ and $\cross(x_2, x_1)_i=1-(x_2)_i=(x_1)_i$, respectively. Together,
\begin{align*}
    & \E(|\cross(x_1, x_2)_i - z_i| + |\cross(x_2, x_1)_i - z_i|)\\
    =\;& |(x_1)_i - z_i|p + |(x_2)_i - z_i|q + |(x_2)_i - z_i|p + |(x_1)_i - z_i|q\\
    =\;& |(x_1)_i - z_i| + |(x_2)_i - z_i|. \qedhere
\end{align*}
\end{proof} 

%
%
%
%

Recall that \wellbehaved operators are respectful by Lemma~\ref{lem:property-one-implies-respectful}. Hence, the following lemma shows that the converse of Lemma~\ref{lem:mask-based-crossovers} is true for \emph{unbiased} crossover operators. In other words, within the class of unbiased binary operators, the properties \wellbehaved and respectful are equivalent. Outside of this class, Lemma~\ref{lem:counterexample-to-respectful-implies-diversity-neutral} shows that the terms are not equivalent.
\begin{lemma}
Every respectful unbiased crossover has an OIM.
\end{lemma}
\begin{proof}
Let $x_1,x_2$ be parents for a respectful, unbiased crossover operator $\cross$ with a corresponding probability distribution $D(y \mid x_1,x_2)$ where the condition is meant to be understood that $x_1$ is the first parent and $x_2$ is the second parent. Let $I_{\textnormal{diff}}$ be the set of components of $x_1,x_2$ which differ, i.e. $I_{\textnormal{diff}}:=\{i \in \{1,\dots,n\} \mid (x_1)_i \neq (x_2)_i\}$. Let $I_{\textnormal{eq}}$ be the set of components of $x_1,x_2$ which are equal, i.e. $I_{\textnormal{eq}}:=\{1,\dots,n\} \setminus I_{\textnormal{diff}}$. 

We show that $\cross$ can be described as a respectful crossover with a mask created according to a probability distribution $M(a, x_1, x_2)$ which is order-independent. For bits $i \in I_{\textnormal{eq}}$ the mask is irrelevant since $\cross$ is respectful, and we (arbitrarily) define $a_i \coloneqq 1$.
For $y \in \{0,1\}^n$ with $D(y \mid x_1,x_2)>0$ choose a mask $a=(a_1,\dots,a_n) \in \{1,2\}^n$ with probability $D(y \mid x_1,x_2)$ in the following way.
For bits $i \in I_{\textnormal{diff}}$ we choose $a_i$ as the unique value from $\{1, 2\}$ such that $(x_{a_i})_i = y_i$. This is possible since $i \in I_{\textnormal{diff}}$ implies $\{(x_1)_i, (x_2)_i)\} = \{0, 1\}$. 
Applying the mask to $x_1$ and $x_2$ creates $y$. Since the corresponding mask is chosen with probability $D(y \mid x_1,x_2)$, each $y$ is created with probability $D(y \mid x_1,x_2)$. Hence $\cross$ is respectful.

It is left to show that the choice of the mask does not depend on the order of the parents for crossover.
Define $w \in \{0,1\}^n$ as $w_i=0$ if $i \in I_{\textnormal{eq}}$ and $w_i=1$ otherwise. Then we obtain $x_1 \oplus w = x_2$ and $x_2 \oplus w = x_1$. Since $\cross$ is unbiased we have
\[
    D(y \mid x_1, x_2) = D(y \oplus w \mid x_1 \oplus w, x_2 \oplus w) = D(y \oplus w \mid x_2, x_1).
\]
So it is left to show the following. Let $a \in \{1,2\}^n$. If we obtain $y \in \{0,1\}^n$ with the mask $a$ applied to $(x_1,x_2)$ then we obtain $y \oplus w$ with the same mask $a$ applied to $(x_2,x_1)$. Let $i \in \{1,\dots,n\}$.

If $i \in I_{\textnormal{eq}}$ then applying the mask $a$ to $(x_1,x_2)$ gives $y_i=(x_1)_i$. Note that $(y \oplus w)_i=y_i = (x_1)_i = (x_2)_i$ which is also the $i$-th component of the offspring if we apply the mask $a$ to $(x_2,x_1)$.

If $i \in I_{\textnormal{diff}}$ then applying the mask $a$ to $(x_1,x_2)$ gives $y_i=(x_{a_i})_i$. If we apply $a$ to $(x_2,x_1)$ we obtain $1-(x_{a_i})_i$ for the $i$-th bit of the offspring, which equals $(y \oplus w)_i$ (since $(x_1)_i$ and $(x_2)_i$ differ).
\end{proof} 

\subsection{Classifying Known Crossover Operators}\label{sec:classification}

We now give examples of \wellbehaved crossover operators, based on~\cite{Friedrich2022}. 
By Lemma~\ref{lem:mask-based-crossovers} it suffices to show that a crossover is respectful with OIM. 
For uniform crossover and $k$-point crossover, this is trivially true as they are based on masks that are chosen independently from the parents. The same holds for the boring crossover (recall that it simply returns one of the parents uniformly at random) as the mask is chosen uniformly from $\{\vec 1, \vec 2\}$.

\emph{Shrinking crossover}~\cite{Chen2006} computes a mask by starting with a window $[\ell, r] = [1, n]$
and then shrinking this window by increasing $\ell$ and/or decreasing $r$ until the substring $x_1[\ell, r]$ has the same number of ones as $x_2[\ell, r]$. Then it swaps these two substrings. The creation of the mask treats both parents symmetrically.

\emph{Balanced uniform crossover}~\cite{Friedrich2022} is respectful as it copies bit values on which both parents agree. If the parents differ in $k$ positions, it chooses values for these bits uniformly at random from all substrings that have exactly $\lfloor k/2 \rfloor$ ones at these positions. The order of parents is irrelevant, hence the crossover has OIM.

Hence, we have shown the following theorem.
\begin{theorem}
The following crossovers 
are \wellbehaved:
\begin{enumerate}[nosep]
\item Uniform crossover with arbitrary crossover bias
\item $k$-point crossover for all $k$
\item Boring crossover
\item Shrinking crossover
\item Balanced uniform crossover 
\end{enumerate}
\end{theorem}

We mention some crossover operators that are not diversity neutral. For details we refer to~\cite{Friedrich2022} and the original papers.

\begin{sloppy}
\emph{Alternating crossover}~\cite{Meinl2009} on $x_1$ and $x_2$ proceeds as follows. If $x_1$ has ones at positions $i_1, \dots, i_k$ and $x_2$ has ones at positions $j_1, \dots, j_{k'}$, then for $k^\ast \coloneqq \min\{k, k'\}$ alternating crossover produces a sorted sequence $s_1, \dots, s_{2k^\ast}$ of these positions. It outputs a search point that has ones at positions $s_1, s_3, s_5, \dots, s_{2k^\ast-1}$.
\end{sloppy}

\emph{Counter-based crossover}~\cite{Manzoni2020} is a variant of uniform crossover ensuring that the offspring has the same number of ones as~$x_1$. It 
creates an offspring bit by bit, choosing values from $x_1$ and $x_2$ uniformly at random, but stopping once the offspring contains $\ones{x_1}$ ones or $\zeros{x_1}$ zeros. In this case a suffix of all-zeros or all-ones, resp., is appended to obtain a bit string of length~$n$ with $\ones{x_1}$ ones.

\emph{Zero length crossover}~\cite{Manzoni2020} uses a different representation: a search point $x$ with $\ones{x}=k$ and $x = 0^{a_1} 1 0^{a_2} 1 \dots 0^{a_k} 1 0^{a_{k+1}}$ is encoded as a vector of runs of zeros: $[a_1, a_2, \dots, a_{k+1}]$. The crossover operator combines encodings from both parents by choosing run lengths in between the run lengths found in both parents.

\emph{Map-of-ones-crossover}~\cite{Manzoni2020} uses an array that contains all indices of 1-bits to represent a bit string. The crossover operator then chooses indices from a randomly chosen parent. In a sense, map-of-ones crossover is a uniform crossover on the map-of-ones representation. 

\emph{Balanced two-point crossover}~\cite{Meinl2009} resembles a two-point crossover on the same representation. It randomly generates two cutting points $u \le v$ and then it takes the first $u-1$ entries of the map-of-ones of $x_1$, the entries at positions $u \dots v$ from the map-of-ones of $x_2$ and the remaining entries from position $v+1$ from $x_1$ again. Any duplicate entries are removed and replaced by entries from the positions $u \dots v$ in the map-of-ones of $x_1$.

\begin{theorem}
The following crossovers are \emph{not} \wellbehaved:
\begin{enumerate}[nosep]
    \item Alternating crossover 
    \item Counter-based crossover
    \item Zero length crossover
    \item Map-of-ones crossover
    \item Balanced two-point crossover
    \item Bit-wise AND and bit-wise OR
\end{enumerate}
\end{theorem}
\begin{proof}
An alternating crossover of $110$ and $101$ creates a sorted sequence of indices $[1, 1, 2, 3]$ and the offspring $110$, irrespective of the order of the parents. For $z \coloneqq 110$, the left-hand side of~\eqref{eq:crossover-property} is $\E(H(\cross(110, 101), 110) + H(c(101, 110), 110)) = \E(2H(110, 110)) = 0$ and the right-hand side is $H(110,110)+H(101,110)=2 \neq 0$.

Crossovers \textit{(2)-(5)} were shown not to be respectful in \citet{Friedrich2022}, thus by the contraposition of Lemma~\ref{lem:property-one-implies-respectful} they are not \wellbehaved.
Bit-wise AND was shown not to be \wellbehaved in the proof of Lemma~\ref{lem:counterexample-to-respectful-implies-diversity-neutral}; bit-wise OR is analogous. 
\end{proof}

\section{Conclusions and Future Work}
\label{sec:conclusions}

We have shown that it is possible to understand the dynamics of population diversity in flat fitness environments in a very general sense, and that it is surprisingly unaffected by most specifics of the algorithm. Of course, our study is only the first step. Possible extensions would include other classes of algorithms like generational GAs
or the effect of diversity-enhancing mechanisms~\cite{sudholt2020benefits} on the dynamics, in particular on the equilibrium state. Note that it is not clear a priori that such a state exists, since the dynamics might be too complex to reduce them to a single number.
Future work could also try to establish connections with population genetics, where the \muea is known as Moran model~\cite{paixao_unified_2015} (cf.\ the discussion at the end of Section~\ref{sec:motivation-for-flat-landscapes}).


The most pressing question is how the dynamics change with selective pressure. We conjectured that for ``reasonable'' situations, the diversity for flat fitness functions is an upper bound on the diversity for non-flat functions. Can this be made precise?
For which non-flat fitness functions can we still characterise how the population diversity evolves over time? These questions have important theoretical and practical implications, yet they are wide open.


\section*{Acknowledgements}
%
This work originated at Dagstuhl seminar 22081 ``Theory of Randomized Optimization Heuristics'' and benefited from Dagstuhl Seminar 22182 ``Estimation-of-Distribution Algorithms: Theory and Applications''.
We thank Jon Rowe and Duc-Cuong Dang for useful discussions and pointers to the literature.

\bibliographystyle{abbrvnat}
\bibliography{references}

\clearpage

\newpage

%
%

\end{document}